\documentclass[preprint,12pt]{elsarticle}

\usepackage{amssymb}     
\usepackage{amsmath}     
\usepackage{amsthm}      
\usepackage{arydshln}
\usepackage{titletoc}
\usepackage{titlesec}
\usepackage{subfigure}
\usepackage{epsfig}
\usepackage{hyperref}    
\usepackage{graphicx}    
\usepackage{algpseudocode}
\usepackage{algorithmicx,algorithm}
\usepackage{colortbl}
\usepackage{tabularx}
\usepackage{booktabs}         
\usepackage{makecell}         
\usepackage{threeparttable}   
\usepackage{multirow}         
\usepackage{cases}            
\usepackage{pifont}           
\usepackage{adjustbox}        

\usepackage{algorithm}   
\usepackage{algorithmicx}
\usepackage{algpseudocode}

\usepackage[title]{appendix}  

\theoremstyle{definition}
\newtheorem{lemma}{Lemma}[section]
\newtheorem{theorem}{Theorem}[section]

\newtheorem{assumption}{Assumption}[section]
\newtheorem{remark}{Remark}[section]
\newtheorem{corollary}{Corollary}[section]
\newtheorem{proposition}{Proposition}[section]

\newcommand{\R}{\mathbb{R}}
\newcommand{\E}{\mathbb{E}}
\newcommand{\PP}{\mathbb{P}}

\journal{Neural Networks}

\begin{document}

\begin{frontmatter}
	
	
	\title{UAdam: Unified Adam-Type Algorithmic Framework for Non-Convex Stochastic Optimization\tnoteref{t1}}

	
	\tnotetext[t1]{This work was funded in part by the National Natural Science Foundation of China (Nos. 62176051, 62272096), in part by National Key R\&D Program of China (No. 2021YFA1003400), and in part by the Fundamental Research Funds for the Central Universities of China (No. 2412020FZ024).}
	
	\author[nenu]{Yiming Jiang}
	\author[nenu]{Jinlan Liu}
	\author[nenu]{Dongpo Xu\corref{cor1}}
	\ead{xudp100@nenu.edu.cn}
	\author[danilo]{Danilo P. Mandic\corref{cor1}}
	\ead{d.mandic@imperial.ac.uk}
	
	\cortext[cor1]{Corresponding authors}
	\address[nenu]{Key Laboratory for Applied Statistics of MOE, School of Mathematics and Statistics, Northeast Normal University, Changchun 130024, P. R. China.}
	\address[danilo]{Department of Electrical and Electronic Engineering, Imperial College London, SW7 2AZ London, UK}

	\begin{abstract}
		Adam-type algorithms have become a preferred choice for optimisation in the deep learning setting, however, despite success, their convergence is still not well understood. To this end, we introduce a unified framework for Adam-type algorithms (called UAdam). This is equipped with a general form of the second-order moment, which makes it possible to include Adam and its variants as special cases, such as NAdam, AMSGrad, AdaBound, AdaFom, and Adan. This is supported by a rigorous convergence analysis of UAdam in the non-convex stochastic setting, showing that UAdam converges to the neighborhood of stationary points with the rate of $\mathcal{O}(1/T)$. Furthermore, the size of neighborhood decreases as $\beta$ increases. Importantly, our analysis only requires the first-order momentum factor to be close enough to 1, without any restrictions on the second-order momentum factor. Theoretical results also show that vanilla Adam can converge by selecting appropriate hyperparameters, which provides a theoretical guarantee for the analysis, applications, and further developments of the whole class of Adam-type algorithms.
	\end{abstract}
	
	
\end{frontmatter}



\section{Introduction}

Deep neural networks have achieved great success in manifold areas including computer vision \cite{computer-version}, image recognition \cite{imagerecognition}, and natural language processing \cite{text-classification,Speech-Emotion-Recognition}. Training of deep neural networks typically considers the following non-convex stochastic optimization setting
\begin{equation}\label{problem}
	\min_{x\in\R^d} f(x)=\E_{\xi\sim\PP}\left[f(x,\xi)\right],
\end{equation}
where $x$ is the model parameter to be optimized, $\xi$ denotes a random variable drawn from some unknown probability distribution $\PP$, and $f(x,\xi)$ designates a differentiable non-convex loss function. The most popular approach to solve the optimization problem in \eqref{problem} is the class of first-order methods based on stochastic gradient. The Stochastic gradient descent (SGD) algorithm \cite{SGD2,SGD1} is widely used due to its simplicity and efficiency, as it updates the model along a negative gradient direction scaled by the stepsize. However, the SGD algorithm may suffer from slow convergence and even become trapped in local minima, particularly for large models and ill-conditioned problems. To address these issues, momentum techniques have been introduced, such as the Heavy Ball (HB) acceleration algorithm proposed by Polyak \cite{HB}, of which the update rule is given by
\begin{equation}\label{eq:HB1}
	\text{SHB: }
	x_{t+1}=x_t-\alpha\nabla f\left(x_t,\xi_t\right)+\beta\left(x_t-x_{t-1}\right),
\end{equation}
where $x_0=x_1\in\R^d$, $\alpha>0$ is the stepsize, and $\beta$ is the (convex) momentum factor, which takes the value $0\leq\beta<1$. Another popular acceleration technique was proposed by Nesterov \cite{NAG} and is referred to as the Nesterov accelerated gradient (NAG), given by \cite{NAG2}
\begin{equation}\label{eq:NAG1}
	\text{SNAG: }
	\left\{
	\begin{split}
		&m_{t}=\beta m_{t-1}-\alpha\nabla f\left(x_t+\beta m_{t-1},\xi_t\right) \\
		&x_{t+1}=x_t+m_t
	\end{split}
	\right.,
\end{equation}
where $\alpha>0$ is the stepsize and $\beta\in[0,1)$ is the momentum factor. Physically, NAG takes a small step from $x_t$ in the direction of the historical gradient, $m_{t-1}$, and utilizes an exponential moving average with the “lookahead gradient” to update the parameters.  It is worth noting that the Nesterov acceleration method converges faster than the Heavy Ball method \cite{Linzhouchen,Nesterov,NAG2}. 

The choice of stepsize is crucial to the convergence performance of SGD, which often requires a fixed stepsize or a sequence of diminishing stepsizes \cite{siam,sun}. Therefore, adaptive learning rate algorithms have emerged as a popular alternative, such as AdaGrad \cite{Adagrad}, RMSProp \cite{RMSprop,rmspropw}, and Adam \cite{Adam}, which adaptively adjust the learning rate based on the second-order moment of historical gradients. For example, AdaGrad \cite{Adagrad} accumulates all past squared gradients element-wise to adjust the stepsize, whereby larger learning rates are assigned to the dimensions with smaller gradients. However, along the iteration process, accumulation can cause excessively small stepsizes, leading to the phenomenon of gradient vanishing. To help resolve this issue, Tieleman and Hinton \cite{RMSprop} proposed the RMSProp algorithm, which replaces the accumulation within the AdaGrad with an exponential moving average. On this basis, Adam \cite{Adam} combines the momentum strategy with RMSProp, and has become is the most popular adaptive method in deep learning. 

Although Adam has performed well in practice, Reddi $et$ $al$. \cite{Reddi} have pointed out that Adam may still be divergent, even for simple convex problems. To this end, researchers have proposed variants of Adam, such as AMSGrad \cite{Reddi}, AdaBound \cite{xuadabound,AdaBound}, AdaFom \cite{Chen}, and Yogi \cite{Yogi}, which differ only in the second-order moments. However, such developments have been rather heuristic. This motivates us to propose a unified framework for the treatment of adaptive momentum algorithms, which incorporates Adam, AdaFom, AMSGrad, AdaBound, and many other algorithms as special cases. Recently, Zhang $et$ $al$. \cite{Adam-without-modification} claimed that Adam can converge without modification of update rules and pointed out that the divergence problem proposed by Reddi $et$ $al$. \cite{Reddi} has a flaw, that is, the parameters are determined first, followed by problem selection. In other words, $\beta$-dependency examples are picked for different pairs of first- and second-order moment parameters $(\beta_1,\beta_2)$ to make Adam diverge. However, in practical applications, the parameter pair $(\beta_1,\beta_2)$ needs to be tuned for a given optimization problem. Therefore, when the problem is given first, Adam can guarantee convergence by appropriately selecting hyperparameters. The convergence results in \cite{Adam-without-modification} require the second moment parameter, $\beta_2$, to be sufficiently large. In contrast, our analysis does not impose any restrictions on the second moment parameter, $\beta_2$, and demonstrates that UAdam, equipped by various forms of the second-order moment, can converge by choosing an appropriate first-order momentum parameter, $\beta_1$, which is consistent with practical observations.

\subsection{Related work}

Although Adam and its variants have achieved remarkable success in training deep neural networks, their theoretical analyses \cite{Adagrad,Adam} have primarily considered online convex settings, and are thus unable to shed light on the convergence in non-convex settings, which are typically encountered in practice. Among the attempts to prove the convergence of Adam and its variants in non-convex settings, Li and Orabona \cite{Li} provided the convergence rate and high probability bound for the generalized global AdaGrad stepsize in a non-convex setting. Moreover, Ward $et$ $al$. \cite{Ward} demonstrated that the norm version of AdaGrad (AdaGrad-Norm) converges to a stationary point at the $\mathcal{O}(\log(T)/\sqrt{T})$ rate in the stochastic setting. Both Chen $et$ $al$. \cite{Chen} and Guo $et$ $al$. \cite{guo} analyzed the convergence performance of a class of Adam algorithms; their analyses are modular and can be extended to solve other optimization problems such as combinatorial and min-max problems \cite{guo}. Zaheer $et$ $al$. \cite{Yogi} studied the impact of minibatch size on the convergence performance of Adam, showing that increasing minibatch sizes facilitates convergence. They also proposed a novel adaptive optimization method (called Yogi) that can control the increase of the effective learning rates so as to achieve better performance. Zou $et$ $al$. \cite{zou} introduced easy-to-check sufficient conditions to ensure the global convergence of Adam and its variants, and provided a new explanation for the divergence of Adam, which may be caused by incorrect setting of second-order moment parameters. D{\'e}fossez $et$ $al$. \cite{1} provided an arbitrarily small upper bound for AdaGrad and Adam, showing that these algorithms can converge at a rate of $\mathcal{O}(d\log(T)/\sqrt{T})$, with an appropriate hyperparameter setting. Furthermore, Zhou $et$ $al$. \cite{Zhou} proved that AMSGrad, modified RMSProp, and AdaGrad converged at a rate of $\mathcal{O}(1/\sqrt{T})$ under the bounded gradient assumption. In addition, Zhang $et$ $al$. \cite{Adam-without-modification} indicated that Adam can converge without modifying the update rules, and does not require bounded gradient or bounded second-order moment assumptions, while Wang $et$ $al$. \cite{L0L1} further analyzed the convergence of Adam under the ($L_0$,$L_1$) smoothness condition. It is worth noting that the above Adam-type algorithms rely on the Heavy Ball method for estimating the first-order moment of the stochastic gradient. 

Recently, the Nesterov acceleration method has been used within both the first- and second-order moment for adaptive learning algorithms. One example is work by Dozat \cite{NAdam} who proposed an adaptive algorithm called NAdam that combines the Nesterov acceleration method and RMSProp. Although NAdam has performed well in practical experiments, there is no supporting theoretical analysis to guarantee convergence. To address this issue, Zou $et$ $al$. \cite{AdaUSM} proposed a new adaptive stochastic momentum algorithm, by combining the weighted coordinate-wise AdaGrad with a unified momentum. They established a non-asymptotic convergence rate of $\mathcal{O}(\log(T)/\sqrt{T})$ under a non-convex setting, thus providing a new perspective on the convergence of Adam and RMSProp. Moreover, Xie $et$ $al$. \cite{Adan} proposed an adaptive Nesterov momentum (Adan) for estimating the first- and second-order moments of the gradient, and proved that Adan requires $\mathcal{O}(\epsilon^{-4})$ stochastic gradient complexity to find an $\epsilon$-stationary point in a non-convex setting.

\begin{table}[h]
	\renewcommand\arraystretch{1}
	\tabcolsep=0.5mm
	\caption{Comparison of different adaptive algorithms. The symbol $\uparrow$ denotes an increase with the iterations or when close enough to 1, $\downarrow$ denotes a decrease with the iterations, ``-'' designates no any restrictions, which `` constant'' means any constant in $[0,1)$. $T$ denotes the number of iterations.}
	\label{table:results}
	\centering
	\scalebox{0.9}{
		\begin{adjustbox}{center}
			\begin{threeparttable}
				\begin{tabular}{c c c c c c}
					\toprule
					\textbf{Optimizer}  & \textbf{Setting} & \textbf{\makecell{First moment\\ parameter}} & \textbf{\makecell{Second moment\\ parameter}} & \textbf{\makecell{Convergence \\ rate}}\\
					\midrule
					Adam \cite{Adam} & convex & $\downarrow$ & constant & no \\
					\midrule
					AMSGrad \cite{Reddi} & convex & non-$\uparrow$ & constant & $\mathcal{O}(1/\sqrt{T})$ \\
					\midrule
					Adam \cite{Chen} & non-convex & non-$\uparrow$ & constant & no \\
					\midrule
					Adam \cite{1,zou} & non-convex & constant & $\uparrow$ & $\mathcal{O}(\ln T/\sqrt{T})$ \\
					\midrule
					Adam \cite{Adam-without-modification} & non-convex & constant & $\uparrow$ & $\mathcal{O}(\ln T/\sqrt{T})+\mathcal{O}(\sqrt{D_0})$\tnote{**} \\
					\midrule
					Adam/Yogi \cite{Yogi} & non-convex & 0 & $\uparrow$ & $\mathcal{O}(1/T+1/b)$\tnote{*} \\
					\midrule
					Adam-style \cite{guo} & non-convex & $\uparrow$ & - & $\mathcal{O}(1/T)+\mathcal{O}(D_0)$\tnote{**} \\
					\midrule
					AMSGrad \cite{Chen} & non-convex & non-$\uparrow$ & $\uparrow$ & $\mathcal{O}(\ln T/\sqrt{T})$ \\
					\midrule
					Adan \cite{Adan} & non-convex & $\uparrow$ & $\uparrow$ & $\mathcal{O}(1/T)+\mathcal{O}(D_0)$\tnote{**} \\
					\midrule
					\textbf{UAdam(ours)}  & non-convex & $\uparrow$ & - & $\mathcal{O}(1/T)+\mathcal{O}(D_0)$\tnote{**} \\
					\bottomrule
				\end{tabular}
				\begin{tablenotes} 
					\footnotesize 
					\item[*] $b$ denotes as mini-batch size. 
					\item[**] $D_0$ is from the weak growth assumption, i.e., $\E_t\left[\left\|g_t-\nabla f(x_t)\right\|^2\right]\leq D_0+D_1\left\|\nabla f(x_t)\right\|^2$. 
				\end{tablenotes} 
			\end{threeparttable}
	\end{adjustbox}}
\end{table}

However, all the above results investigated the convergence of Adam, NAdam, or their variants in a separate form, which makes their comparison difficult and non-obvious. To this end, we here proceed further and introduce a platform for the study of adaptive stochastic momentum algorithms under one general umbrella which encompasses Adam, NAdam, and their variants as special cases. In this way, the proposed unified Adam (UAdam) establishes a unified platform for both the analysis of the existing and the development of future SGD algorithms in deep learning. Finally, Table \ref{table:results} summarizes some existing results and highlights the strengths and the potential of UAdam. 

\subsection{Contributions}

The main contributions of this work can be summarized as follows
\begin{itemize}
	\item We propose a unified framework for the treatment of adaptive stochastic momentum algorithms, called UAdam, which combines the classes of adaptive learning rate and unified momentum methods. The UAdam therefore incorporates existing deep learning optimizers, such as Adam, AMSGrad, NAdam, and Adan as special cases. 
	\item Without any restrictions on the second-order momentum parameter, $\beta_2$, we only need the first-order momentum parameter $\beta_1$ to be close enough to 1 to ensure the convergence of UAdam, which is consistent with the actual hyperparameter settings.
	\item We prove that UAdam can converge to the neighborhood of stationary points with the rate of $\mathcal{O}(1/T)$ in smooth and non-convex settings, and that the size of neighborhood decreases as $\beta$ increases. In addition, under an extra condition (strong growth condition), we can obtain that Adam converges to stationary points. Furthermore, through choice of the interpolation factor $\lambda$, UAdam allows us to immediately obtain the convergence of both Adam-type and NAdam-type algorithms. 
\end{itemize}

The rest of this paper is organized as follows. Section \ref{sec:notation-ass} introduces the notations and assumptions. Section \ref{sec:algorithms} presents a unified framework for the adaptive stochastic momentum algorithms, called UAdam. The technical lemmas and the main convergence results with the rigorous proofs are presented in Section \ref{sec:results}. Finally, this paper concludes with Section \ref{sec:conclusion}.

\section{Preliminaries}\label{sec:notation-ass}
\noindent\textbf{Notations.}
Let $[T]$ be the set $\{1,2,\ldots,T\}$, and denote by $\|\cdot\|$ the $\ell_2$ norm of a vector or the spectral norm of a matrix, if not otherwise specified. For any $t\in[T]$, we use $g_t$ to denote a stochastic gradient of the objective function $f$ at the $t$-th iteration, $x_t$. The symbol $\E_t\left[\cdot\right]$ designates the conditional expectation with respect to $g_t$, conditioned on the past $g_1,g_2,\ldots,g_{t-1}$, while $\E\left[\cdot\right]$ denotes the expectation with respect to the underlying probability space. For any $x_t\in\R^d$, the $i$-th element of $x_t$ is denoted by $x_{t,i}$. All operations on vectors are executed in a coordinate-wise sense, so that for any $x,y\in\R^d,p>0$, $x/y=(x_1/y_1,x_2/y_2,\ldots,x_d/y_d)^T$ and $x^p=(x_1^p,x_2^p,\ldots,x_d^p)^T$.

To analyze the convergence performance of UAdam, we introduce some necessary assumptions.
\begin{assumption}\label{ass:smooth}
	The objective function $f$ is lower bounded by $f_\ast\geq-\infty$ and its gradient $\nabla f$ is $L$-Lipschitz continuous, that is 
	\begin{equation}
		\left\|\nabla f(x)-\nabla f(y)\right\|\leq L\left\|x-y\right\|,~\forall x,y\in\R^d.
	\end{equation}
\end{assumption}

\begin{assumption}\label{ass:unbiased}
	The stochastic gradient $g_t$ is an unbiased estimate of the true gradient $\nabla f(x_t)$, i.e., $\E_t\left[g_t\right]=\nabla f(x_t)$.
\end{assumption}

\begin{assumption}\label{ass:weak}
	The variance of the stochastic gradient $g_t$ satisfies the weak growth condition (WGC), i.e., for some $D_0,D_1>0$, 
	\begin{equation}
		\E_t\left[\left\|g_t-\nabla f(x_t)\right\|^2\right]\leq D_0+D_1\left\|\nabla f(x_t)\right\|^2.
	\end{equation}
\end{assumption}

\begin{remark}
	It is worth noting that the first two assumptions are standard and are frequently used in \cite{Chen,1,Adan,zou}. When $D_1=0$, Assumption \ref{ass:weak} becomes the standard bounded variance condition. Therefore, Assumption \ref{ass:weak} is weaker than the bounded variance condition \cite{Towards,Yogi}. When $D_1\neq0$, the gradient-based algorithms only converge to a bounded neighborhood of stationary points and its neighborhood size is proportional to $D_0$ \cite{better-sgd,Adam-without-modification}. When $D_0=0$, Assumption \ref{ass:weak} is called the strong growth condition (SGC). 
\end{remark}

\section{Unified adaptive stochastic momentum algorithms}\label{sec:algorithms}

\subsection{Stochastic unified momentum algorithms}

Assume that given $x\in\R^d$, it returns a stochastic gradient $\nabla f(x,\xi)$ of the objective function, $f$, defined by the problem \eqref{problem}, where $\xi$ is a random variable satisfying an unknown distribution $\PP$. We use $g_t$ to denote the stochastic gradient $\nabla f(x_t,\xi_t)$ at the $t$-th iteration $x_t$.

By introducing $\eta=\alpha/(1-\beta)$ and $m_t=(x_t-x_{t+1})/\eta$, with $m_0=0$, the stochastic Heavy Ball (SHB) \eqref{eq:HB1} update becomes
\begin{equation}\label{eq:HB2}
	\text{SHB: }
	\left\{
	\begin{split}
		&m_t=\beta m_{t-1}+\left(1-\beta\right) g_t\\
		&x_{t+1}=x_t-\eta m_t
	\end{split}
	\right.,
\end{equation}
Moreover, we consider the form of stochastic NAG (SNAG) given by
\begin{equation}\label{eq:NAG2}
	\text{SNAG: }
	\left\{
	\begin{split}
		&m_t=\beta m_{t-1}+(1-\beta) g_t \\
		&x_{t+1} = x_{t}-\eta\beta m_t-\eta(1-\beta) g_t
	\end{split}
	\right.,
\end{equation}
which is equivalent to SNAG \eqref{eq:NAG1} but easier to analyze (see Proposition \ref{pro:NAG} for the proof). 
Therefore, the updates of SHB in \eqref{eq:HB2} and SNAG in \eqref{eq:NAG2} can be written in the form of the stochastic unified momentum (SUM) as follows
\begin{equation}\label{eq:SUM1}
	\text{SUM}_1\text{: }
	\left\{
	\begin{split}
		&m_t=\beta m_{t-1}+(1-\beta) g_t \\
		&\bar{m}_t=m_t-\tilde{\lambda} (m_t-g_t) \\
		&x_{t+1} = x_{t}-\eta \bar{m}_t
	\end{split}
	\right.,
\end{equation}
where $\tilde{\lambda}=(1-\beta)\lambda\in[0,1]$, $\lambda\in[0,1/(1-\beta)]$ is a interpolation factor, and $\beta$ is a momentum parameter. Observe that when $\lambda=0$, $\tilde{\lambda}=0$, SUM in \eqref{eq:SUM1} becomes SHB in \eqref{eq:HB2}; when $\lambda=1$, $\tilde{\lambda}=1-\beta$, SUM in \eqref{eq:SUM1} becomes SNAG in \eqref{eq:NAG2}; when $\lambda=1/(1-\beta)$, $\tilde{\lambda}=1$, SUM in \eqref{eq:SUM1} becomes SGD.

\begin{remark}
	Xie $et$ $al$. \cite{Adan} developed a Nesterov momentum estimation (NME) method to estimate the first-order moment of the stochastic gradient, with the update given by
	\begin{equation}\label{eq:NME}
		\text{NME: }
		\left\{
		\begin{split}
			&\bar{m}_t=\beta \bar{m}_{t-1}+(1-\beta)\left(g_t+\beta\left(g_t-g_{t-1}\right)\right)\\
			&x_{t+1}=x_t-\eta \bar{m}_t
		\end{split}
		\right.,
	\end{equation}
	It is worth mentioning that the above NME is fundamentally equivalent to SNAG in \eqref{eq:NAG2}. A detailed proof of the equivalence is given in Proposition \ref{pro:NME}.
\end{remark}

\begin{remark}
	Liu $et$ $al$. \cite{SUM} unified SHB and SNAG in the following form
	\begin{equation}\label{eq:SUM2}
		\text{SUM}_2\text{: }
		\left\{
		\begin{split}
			&m_t=\mu m_{t-1}-\eta_t g_t\\
			&x_{t+1} = x_{t}-\lambda\eta_t g_t+(1-\tilde{\lambda})m_t
		\end{split}
		\right.,
	\end{equation}
	where $\tilde{\lambda}:=(1-\mu)\lambda$. Notice that SUM in \eqref{eq:SUM2} is functionally equivalent to SUM in \eqref{eq:SUM1}, except for some parameter variations. More detailed information on this observation is provided in Proposition \ref{pro:SUM}.
\end{remark}

\subsection{Adaptive learning rate}

Next, we investigate adaptive learning rates with the bounded assumption, which can cover a large class of adaptive gradient algorithms, as shown in Table \ref{table:1}. 

\begin{assumption}\label{ass:bound}
	The adaptive learning rate, $\eta_t$, is upper bounded and lower bounded, i.e., there exists $0<\eta_l<\eta_u$, such that $\forall i\in[d],~\eta_l\leq\eta_{t,i}\leq\eta_u$, where $\eta_{t,i}$ denotes the $i$-th component of $\eta_t$.
\end{assumption}

\begin{table}[H]
	\caption{Forms of the learning rate, $\eta_t$, and their compliance with Assumption \ref{ass:bound}}
	\label{table:1}
	\centering
	\resizebox{\textwidth}{!}{
		\begin{tabular}{c c c c}
			\toprule
			\textbf{Optimizer}&\textbf{Learning rate $\eta_t$} & \textbf{\makecell{Additional\\ assumption}} & \textbf{$\eta_l$ and $\eta_u$}\\
			\midrule
			SUM \cite{SUM}&$\eta_{t}=\eta$ &-& $\eta_l=\eta,\eta_u=\eta$\\
			\midrule
			Adam \cite{Adam}&$\makecell{v_t=\beta_2 v_{t-1}+\left(1-\beta_2\right)g_t^2,\\ \eta_t=\eta/\sqrt{v_t}+\epsilon}$ & $\|g_t\|_\infty\leq G$ & $\eta_l=\frac{\eta}{G+\varepsilon},\eta_u=\frac{\eta}{\varepsilon}$ \\
			\midrule
			AMSGrad \cite{Reddi}&$\makecell{\bar{v}_t=\beta_2 \bar{v}_{t-1}+\left(1-\beta_2\right)g_t^2,\\v_t=\max(v_{t-1},\bar{v}_t),\eta_t=\eta/\sqrt{v_t}+\epsilon}$ & $\|g_t\|_\infty\leq G$ & $\eta_l=\frac{\eta}{G+\varepsilon},\eta_u=\frac{\eta}{\varepsilon}$ \\
			\midrule
			AdaFom \cite{Chen}&$v_t=\frac{1}{t}\sum\limits_{i=1}^{t}g_i^2,\eta_t=\eta/\sqrt{v_t}+\epsilon$ & $\|g_t\|_\infty\leq G$ & $\eta_l=\frac{\eta}{G+\varepsilon},\eta_u=\frac{\eta}{\varepsilon}$ \\
			\midrule
			AdaBound \cite{AdaBound}&$\makecell{\bar{v}_t=\beta_2 \bar{v}_{t-1}+\left(1-\beta_2\right)g_t^2,\\v_t=\text{Clip}(\bar{v}_t,1/c_u^2,1/c_l^2),\eta_t=\eta/\sqrt{v_t}}$ & - & $\eta_l=\eta c_l,\eta_u=\eta c_u$ \\
			\midrule
			Yogi \cite{Yogi}&$\makecell{v_t=v_{t-1}-\left(1-\beta_2\right)\text{sign}(v_{t-1}-g_t^2)g_t^2,\\ \eta_t=\eta/\sqrt{v_t}+\epsilon}$ & $\|g_t\|_\infty\leq G$ & $\eta_l=
			\frac{\eta}{\sqrt{2}G+\varepsilon},\eta_u=\frac{\eta}{\varepsilon}$ \\
			\midrule
			AdaEMA \cite{zou}&$\makecell{v_t=\frac{1}{W_t}\sum\limits_{i=1}^{t}\omega_ig_i^2, W_t=\sum\limits_{i=1}^{t}\omega_i,\\ \eta_t=\eta/\sqrt{v_t}+\epsilon}$ & $\|g_t\|_\infty\leq G$ & $\eta_l=\frac{\eta}{G+\varepsilon},\eta_u=\frac{\eta}{\varepsilon}$ \\
			\midrule
			Adan \cite{Adan}&$\makecell{v_t=(1-\beta_2)v_{t-1}\\+\beta_2(g_t+(1-\beta_1)(g_t-g_{t-1}))^2,\\ \eta_t=\eta/\sqrt{v_t}+\epsilon}$ & $\|g_t\|_\infty\leq G/3$ & $\eta_l=
			\frac{\eta}{G+\varepsilon},\eta_u=\frac{\eta}{\varepsilon}$ \\
			\midrule
			SAdam \cite{sadam}&$\makecell{v_t=\beta_2 v_{t-1}+\left(1-\beta_2\right)g_t^2,\\ \text{softplus}(x)=\log(1+\exp(\theta x))/\theta, \\ \eta_t=\eta/\text{softplus}(\sqrt{v_t}) }$ & $\|g_t\|_\infty\leq G$ & \makecell{$\eta_l=
				\frac{\eta\theta}{\log(1+\exp(\theta G))}$, \\ $\eta_u=\frac{\eta\theta}{\log2}$} \\
			\bottomrule
	\end{tabular}}
\end{table}

\begin{remark}
	Under the bounded stochastic gradient condition, Adam and its variants, such as AMSGrad \cite{Reddi}, AdaEMA \cite{zou}, and Adan \cite{Adan} can satisfy Assumption \ref{ass:bound}. Even when the boundedness condition is not satisfied, we can still use the clipping technique \cite{AdaBound} to make Assumption \ref{ass:bound} hold. We emphasize that the additional assumption $\|g_t\|_\infty\leq G$ is often required in the convergence analysis of the Adam-type algorithms \cite{Chen,1,Yogi,zou}.
\end{remark}

\begin{remark}
	Existing convergence analyses \cite{1,Adam-without-modification,zou} require the second-order momentum factor, $\beta_2$, to be close to 1 to guarantee the convergence of Adam. In contrast, we do not impose any restrictions on $\beta_2$, and only need boundedness of stochastic gradients to satisfy Assumption \ref{ass:bound}.
\end{remark}

\subsection{UAdam: Unified adaptive stochastic momentum algorithm}

In this section, we present a unified framework for adaptive stochastic momentum algorithm, termed UAdam, which effectively integrates SUM in \eqref{eq:SUM1} with a class of adaptive learning rate methods satisfying Assumption \ref{ass:bound}. The pseudocode for the UAdam algorithm is given in Algorithm \ref{UAdam}.

\begin{algorithm}[H]
	\caption{UAdam: Unified Adaptive Stochastic Momentum Algorithm}
	\label{UAdam}
	\begin{algorithmic}[1]
		\Require
		First-order moment factor $\beta\in[0,1)$, interpolation factor $\lambda\in[0,1/(1-\beta)],\tilde{\lambda}=(1-\beta)\lambda$.
		\Ensure
		$x_1\in\R^d,~m_0=0$
		\For{$t=1,2,\ldots,T$}
		\State Sample an unbiased stochastic gradient estimator: $g_t=\nabla f(x_t,\xi_t)$
		\State $m_t=\beta m_{t-1}+(1-\beta) g_t$
		\State $\bar{m}_t=m_t-\tilde{\lambda} (m_t-g_t)$
		\State $\eta_t=h_t\left(g_1,g_2,\ldots,g_t\right)$ (See different forms of $\eta_t$ in Table \ref{table:1})
		\State $x_{t+1} = x_{t}-\eta_t \bar{m}_t$
		\EndFor
	\end{algorithmic}
\end{algorithm}

\begin{remark}\label{remark:1}		
	Notice that when the interpolation factor, $\lambda$, and the adaptive learning rate, $\eta_t$, take different forms, UAdam corresponds to different deep learning algorithms. For example, if the learning rate is taken as $\eta_t=\eta/\sqrt{v_t}+\epsilon$ with $v_t=\beta_2 v_{t-1}+\left(1-\beta_2\right)g_t^2$, then for $\lambda=0$, UAdam degenerates into the original Adam, while when $\lambda=1$, UAdam becomes NAdam \cite{NAdam}. Alternatively, if the learning rate is taken as $\eta_t=\eta/\sqrt{v_t}+\epsilon$ with $\bar{v}_t=\beta_2 \bar{v}_{t-1}+\left(1-\beta_2\right)g_t^2$, $v_t=\max(v_{t-1},\bar{v}_t)$ and $\lambda=0$, UAdam becomes AMSGrad \cite{Reddi}. Similarly, when the learning rate is taken as $\eta_t=\eta/\sqrt{v_t}+\epsilon$ with $v_t=(1-\beta_2) v_{t-1}+\beta_2\left(g_t+\left(1-\beta_1\right)(g_t-g_{t-1})\right)^2$ and $\lambda=1$, it follows from Proposition \ref{pro:NME} that UAdam becomes Adan \cite{Adan}.
\end{remark}

\section{Main results}\label{sec:results}

\subsection{Technical lemmas}\label{sec:lemma}

We next provide some lemmas, which play an essential role in the convergence analysis of UAdam. The key to the convergence analysis is the estimation of the variance of stochastic exponential moving average sequences, as shown in Lemma \ref{le:VR}.

\begin{lemma}\label{le:VR}
	Consider a stochastic exponential moving average sequence, $m_t=\beta_tm_{t-1}+\left(1-\beta_t\right)g_t$, where $0\leq\beta_t<1$. Suppose that Assumptions \ref{ass:smooth} and \ref{ass:unbiased} hold. Then, 
	\begin{equation}\label{our}
		\begin{split}
			\E_t\left[\left\|m_t-\nabla f(x_t)\right\|^2\right]
			\leq&\beta_t\left\|m_{t-1}-\nabla f(x_{t-1})\right\|^2+\frac{\beta_t^2}{1-\beta_t}L^2\left\|x_{t}-x_{t-1}\right\|^2\\
			&+\left(1-\beta_t\right)^2\E_t\left[\left\|g_t-\nabla f(x_{t})\right\|^2\right].
		\end{split}
	\end{equation}
	
\end{lemma}

\begin{proof}
	According to the definition of $m_t$, we have
	\begin{equation}\label{eq:1}
		\begin{split}
			m_t-\nabla f(x_t)&=\beta_tm_{t-1}+\left(1-\beta_t\right)g_t-\nabla f(x_{t})\\
			&=\beta_t\left(m_{t-1}-\nabla f(x_{t-1})\right)+\left(1-\beta_t\right)\left(g_t-\nabla f(x_{t})\right)\\&\quad+\beta_t\left(\nabla f(x_{t-1})-\nabla f(x_{t})\right).
		\end{split}
	\end{equation}
	Upon taking the squared norm of \eqref{eq:1}, we obtain
	\begin{equation}\label{eq:2}
		\begin{split}
			&\left\|m_t-\nabla f(x_t)\right\|^2
			=\beta_t^2\left\|m_{t-1}-\nabla f(x_{t-1})\right\|^2 \\&\qquad\qquad\qquad\; +\beta_t^2\left\|\nabla f(x_{t-1})-\nabla f(x_{t})\right\|^2 + \left(1-\beta_t\right)^2\left\|g_t-\nabla f(x_{t})\right\|^2 \\&\qquad\qquad\qquad\; +2\beta_t\left(1-\beta_t\right)\underbrace{\left\langle m_{t-1}-\nabla f(x_{t}),g_t-\nabla f(x_{t}) \right\rangle}_\spadesuit \\&\qquad\qquad\qquad\; +2\beta_t^2\underbrace{\left\langle m_{t-1}-\nabla f(x_{t-1}),\nabla f(x_{t-1})-\nabla f(x_{t}) \right\rangle}_\clubsuit.
		\end{split}
	\end{equation}
	For the term $\spadesuit$ in \eqref{eq:2}, upon taking the conditional expectation, under the condition that $g_1,\ldots,g_{t-1}$ are known, then both $x_t$ and $m_{t-1}$ are measurable. Since $\E_t\left[g_t\right]=\nabla f(x_t)$ by Assumption \ref{ass:unbiased}, we further obtain
	\begin{equation}\label{eq:2-1}
		\E_t\left[\spadesuit\right]=\left\langle m_{t-1}-\nabla f(x_{t}),\E_t\left[g_t\right]-\nabla f(x_{t}) \right\rangle=0.
	\end{equation}
	Regarding the term $\clubsuit$ in \eqref{eq:2}, from the fact that $\left\langle a,b\right\rangle \leq\epsilon/2\left\|a\right\|^2+1/2\epsilon\left\|b\right\|^2$, $\forall\epsilon>0$, let $a=m_{t-1}-\nabla f(x_{t-1}),b=\nabla f(x_{t-1})-\nabla f(x_{t})$, and $\epsilon=(1-\beta_t)/\beta_t$. It then follows that 
	\begin{equation}\label{eq:2-2}
		\clubsuit\leq\frac{1-\beta_t}{2\beta_t}\left\|m_{t-1}-\nabla f(x_{t-1})\right\|^2+\frac{\beta_t}{2(1-\beta_t)}\left\|\nabla f(x_{t-1})-\nabla f(x_{t})\right\|^2.
	\end{equation}
	Upon taking the conditional expectation of \eqref{eq:2}, and inserting \eqref{eq:2-1} and \eqref{eq:2-2} into \eqref{eq:2}, and under the condition that $g_1,\ldots,g_{t-1}$ are known, and that $x_t$, $x_{t-1}$, and $m_{t-1}$ are measurable, we have
	\begin{equation}
		\begin{split}
			\E_t\left[\left\|m_t-\nabla f(x_t)\right\|^2\right]
			\leq&\beta_t\left\|m_{t-1}-\nabla f(x_{t-1})\right\|^2\\&+\frac{\beta_t^2}{1-\beta_t}\left\|\nabla f(x_{t-1})-\nabla f(x_{t})\right\|^2\\&+\left(1-\beta_t\right)^2\E_t\left[\left\|g_t-\nabla f(x_{t})\right\|^2\right].
		\end{split}
	\end{equation}
	Using the fact that $\nabla f$ is $L$-Lipschitz continuous, we arrive at
	\begin{equation}
		\begin{split}
			\E_t\left[\left\|m_t-\nabla f(x_t)\right\|^2\right]
			\leq&\beta_t\left\|m_{t-1}-\nabla f(x_{t-1})\right\|^2\\&+\frac{\beta_t^2}{1-\beta_t}L^2\left\|x_{t}-x_{t-1}\right\|^2\\&+\left(1-\beta_t\right)^2\E_t\left[\left\|g_t-\nabla f(x_{t})\right\|^2\right].
		\end{split}
	\end{equation}
	This completes the proof.
\end{proof}

\begin{remark}
	By replacing $\beta_t$ with $1-\beta_t$, we obtain an equivalent form of Lemma \ref{le:VR} 
	\begin{equation}\label{wang}
		\begin{split}
			&\E_t\left[\left\|m_t-\nabla f(x_t)\right\|^2\right]
			\leq\left(1-\beta_t\right)\left\|m_{t-1}-\nabla f(x_{t-1})\right\|^2 \\&\qquad\qquad\qquad\quad +\frac{\left(1-\beta_t\right)^2}{\beta_t}L^2\left\|x_{t}-x_{t-1}\right\|^2+\beta_t^2\E_t\left[\left\|g_t-\nabla f(x_{t})\right\|^2\right].
		\end{split}
	\end{equation}
	It follows from $\beta_t\in[0,1)$ that our estimation \eqref{wang} is tighter than Lemma 2 in \cite{Wangmengdi}.
\end{remark}

\begin{lemma}\label{le:delta}
	Let $x_t$ be the iteration sequence generated by the UAdam algorithm. Suppose that Assumptions \ref{ass:smooth}, \ref{ass:unbiased}, \ref{ass:weak} and \ref{ass:bound} hold. Then, 
	\begin{equation}
		\begin{split}
			&\E\left[\sum_{t=1}^T\Delta_{t}\right]\leq\E\left[\frac{\Delta_{1}-\Delta_{T+1}}{1-\beta}\right]+2\left(1-\beta\right)D_1\E\left[\sum_{t=1}^T\left\|\nabla f(x_t)\right\|^2\right]\\&\qquad+\left(1-\beta\right)D_0T+\tilde{\lambda}\left(2\left(1-\beta\right)D_1L^2\eta_u^2+\frac{\beta^2L^2\eta_u^2}{(1-\beta)^2}\right)\E\left[\sum_{t=1}^T\left\|g_t\right\|^2\right]\\&\qquad+(1-\tilde{\lambda})\left(2\left(1-\beta\right)D_1L^2\eta_u^2+\frac{\beta^2L^2\eta_u^2}{(1-\beta)^2}\right)\E\left[\sum_{t=1}^T\left\|m_{t}\right\|^2\right].
		\end{split}
	\end{equation}
	where $\tilde{\lambda}=(1-\beta)\lambda\in[0,1]$, $m_t=\beta m_{t-1}+\left(1-\beta\right) g_t$, $0\leq\beta<1$, and $\Delta_t=\left\|m_t-\nabla f(x_t)\right\|^2$.
\end{lemma}

\begin{proof}
	Let $\Delta_t=\left\|m_t-\nabla f(x_t)\right\|^2$. Then, upon applying Lemma \ref{le:VR} with $\beta_t=\beta\in[0,1)$, and taking the total expectation, it follows that
	\begin{equation}\label{eq:4}
		\begin{split}
			\E\left[\Delta_{t+1}\right] \leq& \beta\E\left[\Delta_{t}\right]+\frac{\beta^2L^2}{1-\beta}\E\left[\left\|x_{t+1}-x_{t}\right\|^2\right]\\&+\left(1-\beta\right)^2\E\left[\left\|g_{t+1}-\nabla f(x_{t+1})\right\|^2\right].
		\end{split}
	\end{equation}
	Upon rearranging the terms in \eqref{eq:4}, we have
	\begin{equation}\label{eq:3}
		\begin{split}
			\E\left[\Delta_{t}\right]\leq&\frac{\E\left[\Delta_{t}-\Delta_{t+1}\right]}{1-\beta}+\frac{\beta^2L^2}{(1-\beta)^2}\E\left[\left\|x_{t+1}-x_{t}\right\|^2\right]\\&+(1-\beta)\E\left[\left\|g_{t+1}-\nabla f(x_{t+1})\right\|^2\right].\\
		\end{split}
	\end{equation}
	For the second term on the right-hand side of \eqref{eq:3}, according to the iteration, $x_{t+1} = x_{t}-\eta_t\bar{m}_t,\bar{m}_t=m_t-\tilde{\lambda} (m_t-g_t)$ in Algorithm \ref{UAdam}, we have
	\begin{equation}\label{eq:3-1}
		\begin{split}
			\left\|x_{t+1}-x_{t}\right\|^2&=\left\|\eta_t\bar{m}_t\right\|^2=\left\|\tilde{\lambda}\eta_t g_t+(1-\tilde{\lambda})\eta_t m_t\right\|^2 \\
			&\overset{(i)}{\leq} \tilde{\lambda}\|\eta_tg_t\|^2+(1-\tilde{\lambda})\|\eta_tm_{t}\|^2 \\
			&\overset{(ii)}{\leq} \tilde{\lambda}\eta_u^2\left\|g_t\right\|^2+(1-\tilde{\lambda})\eta_u^2\left\|m_{t}\right\|^2,
		\end{split}
	\end{equation}
	where ($i$) uses the convexity of $\|\cdot\|^2$ and ($ii$) follows from Assumption \ref{ass:bound}, $\eta_l\leq\eta_{t,i}\leq\eta_u$. 
	For the third term on the right-hand side of \eqref{eq:3}, according to Assumption \ref{ass:weak}, we have
	\begin{equation}\label{eq:3-2-1}
		\begin{split}
			\E\left[\left\|g_{t+1}-\nabla f(x_{t+1})\right\|^2\right] \leq D_0+D_1\E\left[\left\|\nabla f(x_{t+1})\right\|^2\right].
		\end{split}
	\end{equation}
	For the second term on the right-hand side of \eqref{eq:3-2-1}, since $\nabla f$ is $L$-Lipschitz continuous, we obtain
	\begin{equation}\label{eq:3-2-2}
		\begin{split}
			&\E\left[\left\|\nabla f(x_{t+1})\right\|^2\right] = \E\left[\left\|\nabla f(x_{t+1})-\nabla f(x_t)+\nabla f(x_t)\right\|^2\right]\\
			&\quad\,\leq2\E\left[\left\|\nabla f(x_{t+1})-\nabla f(x_t)\right\|^2\right]+2\E\left[\left\|\nabla f(x_t)\right\|^2\right]\\
			&\quad\,\leq2L^2\E\left[\left\|x_{t+1}-x_t\right\|^2\right]+2\E\left[\left\|\nabla f(x_t)\right\|^2\right]\\
			&\quad\overset{\eqref{eq:3-1}}{\leq}2\tilde{\lambda} L^2\eta_u^2\E\left[\left\|g_t\right\|^2\right]+2(1-\tilde{\lambda})L^2\eta_u^2\E\left[\left\|m_{t}\right\|^2\right]+2\E\left[\left\|\nabla f(x_t)\right\|^2\right].
		\end{split}
	\end{equation}
	Upon combining \eqref{eq:3-2-1} and \eqref{eq:3-2-2}, we obtain
	\begin{equation}\label{eq:3-2}
		\begin{split}
			&\E\left[\left\|g_{t+1}-\nabla f(x_{t+1})\right\|^2\right] \leq D_0+2D_1\E\left[\left\|\nabla f(x_t)\right\|^2\right]\\&\qquad\qquad\qquad\qquad+2\tilde{\lambda} D_1L^2\eta_u^2\E\left[\left\|g_t\right\|^2\right]+2(1-\tilde{\lambda})D_1L^2\eta_u^2\E\left[\left\|m_{t}\right\|^2\right].
		\end{split}
	\end{equation}
	A substitution of \eqref{eq:3-1} and \eqref{eq:3-2} into \eqref{eq:3} yields
	\begin{equation}
		\begin{split}
			\E\left[\Delta_{t}\right]\leq&\frac{\E\left[\Delta_{t}-\Delta_{t+1}\right]}{1-\beta}+\left(1-\beta\right)D_0+2\left(1-\beta\right)D_1\E\left[\left\|\nabla f(x_t)\right\|^2\right]\\&+\tilde{\lambda}\left(2\left(1-\beta\right)D_1L^2\eta_u^2+\frac{\beta^2L^2\eta_u^2}{(1-\beta)^2}\right)\E\left[\left\|g_t\right\|^2\right]\\&+(1-\tilde{\lambda})\left(2\left(1-\beta\right)D_1L^2\eta_u^2+\frac{\beta^2L^2\eta_u^2}{(1-\beta)^2}\right)\E\left[\left\|m_{t}\right\|^2\right].
		\end{split}
	\end{equation}
	Upon summing up the above inequality for all iterations $t\in[T]$, we obtain
	\begin{equation}
		\begin{split}
			&\E\left[\sum_{t=1}^T\Delta_{t}\right]\leq\E\left[\frac{\Delta_{1}-\Delta_{T+1}}{1-\beta}\right]+2\left(1-\beta\right)D_1\E\left[\sum_{t=1}^T\left\|\nabla f(x_t)\right\|^2\right]\\&\qquad+\left(1-\beta\right)D_0T+\tilde{\lambda}\left(2\left(1-\beta\right)D_1L^2\eta_u^2+\frac{\beta^2L^2\eta_u^2}{(1-\beta)^2}\right)\E\left[\sum_{t=1}^T\left\|g_t\right\|^2\right]\\&\qquad+(1-\tilde{\lambda})\left(2\left(1-\beta\right)D_1L^2\eta_u^2+\frac{\beta^2L^2\eta_u^2}{(1-\beta)^2}\right)\E\left[\sum_{t=1}^T\left\|m_{t}\right\|^2\right].
		\end{split}
	\end{equation}
	This completes the proof.
\end{proof}

\begin{lemma}\label{le:f}
	Let $x_t$ be the iteration sequence generated by the UAdam algorithm. Suppose that Assumption \ref{ass:smooth} is satisfied. Then, 
	\begin{equation}
		\begin{split}
			f(x_{t+1}) \leq&   f(x_t)+\frac{\tilde{\lambda}}{2}\left\|\sqrt{\eta_t}\left(g_{t}-\nabla f(x_t)\right)\right\|^2+\frac{1-\tilde{\lambda}}{2}\left\|\sqrt{\eta_t}\left(m_{t}-\nabla f(x_t)\right)\right\|^2\\&-\frac{1}{2}\left\|\sqrt{\eta_t}\nabla f(x_t)\right\|^2-\frac{\tilde{\lambda}}{2}\left\|\sqrt{\eta_t}g_{t}\right\|^2+\frac{\tilde{\lambda} L}{2}\left\|\eta_tg_t\right\|^{2}\\&-\frac{1-\tilde{\lambda}}{2}\left\|\sqrt{\eta_t}m_{t}\right\|^2+\frac{(1-\tilde{\lambda})L}{2}\left\|\eta_tm_{t}\right\|^{2},
		\end{split}
	\end{equation}
	where $\tilde{\lambda}=(1-\beta)\lambda\in[0,1]$ and $\beta\in[0,1)$. 
\end{lemma}

\begin{proof}
	Since the gradient $\nabla f$ is $L$-Lipschitz continuous, according to the Descent Lemma \cite{Nesterov}, and the iteration of the UAdam algorithm, $x_{t+1} = x_{t}-\eta_t\bar{m}_t,\bar{m}_t=m_t-\tilde{\lambda} (m_t-g_t)$, we have
	\begin{equation}
		\begin{split}
			f(x_{t+1}) \leq& f(x_t)+\langle \nabla f(x_t),x_{t+1}-x_t \rangle+\frac{L}{2}\left\|x_{t+1}-x_t\right\|^{2}\\
			=&f(x_t)-\langle \nabla f(x_t),\eta_t\bar{m}_t \rangle+\frac{L}{2}\left\|\eta_t\bar{m}_t\right\|^{2}\\
			=& f(x_t)-\left\langle \nabla f(x_t),\eta_t\left(m_t-\tilde{\lambda} (m_t-g_t)\right) \right\rangle\\&+\frac{L}{2}\left\|\eta_t\left(m_t-\tilde{\lambda} (m_t-g_t)\right)\right\|^{2}\\
			=& f(x_t)-\tilde{\lambda}\langle \nabla f(x_t),\eta_tg_t \rangle-(1-\tilde{\lambda})\langle \nabla f(x_t),\eta_tm_{t} \rangle\\&+\frac{L}{2}\left\|\tilde{\lambda}\eta_t g_t+(1-\tilde{\lambda})\eta_tm_{t}\right\|^{2}.
		\end{split}
	\end{equation}
	From the convexity of $\|\cdot\|^2$: $\|\tilde{\lambda} x+(1-\tilde{\lambda})y\|^2\leq\tilde{\lambda}\|x\|^2+(1-\tilde{\lambda})\|y\|^2$ and the fact that $-2\langle a,b\rangle=\|a-b\|^2-\|a\|^2-\|b\|^2$, we have
	\begin{equation}
		\begin{split}
			f(x_{t+1}) \leq&  f(x_t)+\frac{\tilde{\lambda}}{2}\left\|\sqrt{\eta_t}\left(g_{t}-\nabla f(x_t)\right)\right\|^2-\frac{\tilde{\lambda}}{2}\left\|\sqrt{\eta_t}\nabla f(x_t)\right\|^2-\frac{\tilde{\lambda}}{2}\left\|\sqrt{\eta_t}g_{t}\right\|^2\\&+\frac{1-\tilde{\lambda}}{2}\left\|\sqrt{\eta_t}\left(m_{t}-\nabla f(x_t)\right)\right\|^2-\frac{1-\tilde{\lambda}}{2}\left\|\sqrt{\eta_t}\nabla f(x_t)\right\|^2\\&-\frac{1-\tilde{\lambda}}{2}\left\|\sqrt{\eta_t}m_{t}\right\|^2+\frac{\tilde{\lambda} L}{2}\left\|\eta_tg_t\right\|^{2}+\frac{(1-\tilde{\lambda})L}{2}\left\|\eta_tm_{t}\right\|^{2}\\
			=&  f(x_t)+\frac{\tilde{\lambda}}{2}\left\|\sqrt{\eta_t}\left(g_{t}-\nabla f(x_t)\right)\right\|^2+\frac{1-\tilde{\lambda}}{2}\left\|\sqrt{\eta_t}\left(m_{t}-\nabla f(x_t)\right)\right\|^2\\&-\frac{1}{2}\left\|\sqrt{\eta_t}\nabla f(x_t)\right\|^2-\frac{\tilde{\lambda}}{2}\left\|\sqrt{\eta_t}g_{t}\right\|^2+\frac{\tilde{\lambda} L}{2}\left\|\eta_tg_t\right\|^{2}\\&-\frac{1-\tilde{\lambda}}{2}\left\|\sqrt{\eta_t}m_{t}\right\|^2+\frac{(1-\tilde{\lambda})L}{2}\left\|\eta_tm_{t}\right\|^{2}.
		\end{split}
	\end{equation}
	This completes the proof.
\end{proof}

\subsection{Convergence analysis of UAdam for non-convex optimization}

With the help of the lemmas in Section \ref{sec:lemma}, we now proceed to establish the convergence analysis of the UAdam algorithm.

\begin{theorem}\label{th:UAdam}
	Let $x_t$ be the iteration sequence generated by UAdam. Suppose that Assumptions \ref{ass:smooth}, \ref{ass:unbiased}, \ref{ass:weak} and \ref{ass:bound} are satisfied. With $0<1-\beta\leq\min\left\{\frac{\eta_l}{2(2+\lambda)D_1\eta_u},1\right\}$ and $\eta_u\leq\min\left\{\frac{\sqrt[3]{\eta_l}}{2\sqrt[3]{D_1L^2}}, \sqrt[3]{\frac{\left(1-\beta\right)^2\eta_l}{4L^2}},\sqrt{\frac{\eta_l}{2L}}\right\}$, we have
	\begin{equation}
		\begin{split}
			\frac{1}{T}\E\left[\sum_{t=1}^T\left\|\nabla f(x_t)\right\|^2\right]\leq\mathcal{O}\left(\frac{1}{T}\right)+\mathcal{O}\big((1-\beta)D_0\big).
		\end{split}
	\end{equation}
\end{theorem}

\begin{proof}
	According to Lemma \ref{le:f}, using Assumption \ref{ass:bound}, $\eta_l\leq\eta_{t,i}\leq\eta_u$, and $\eta_u^2\leq\eta_l/(2L)$, we have
	\begin{equation}\label{eq:f}
		\begin{split}
			f(x_{t+1}) \leq& f(x_t)+\frac{\tilde{\lambda}\eta_u}{2}\left\|g_{t}-\nabla f(x_t)\right\|^2+\frac{(1-\tilde{\lambda})\eta_u}{2}\left\|m_{t}-\nabla f(x_t)\right\|^2\\&-\frac{\eta_l}{2}\left\|\nabla f(x_t)\right\|^2+\frac{\tilde{\lambda}(L\eta_u^2-\eta_l)}{2}\left\|g_t\right\|^{2}+\frac{(1-\tilde{\lambda})(L\eta_u^2-\eta_l)}{2}\left\|m_{t}\right\|^2\\
			\leq& f(x_t)+\frac{\tilde{\lambda}\eta_u}{2}\left\|g_{t}-\nabla f(x_t)\right\|^2+\frac{(1-\tilde{\lambda})\eta_u}{2}\left\|m_{t}-\nabla f(x_t)\right\|^2\\&-\frac{\eta_l}{2}\left\|\nabla f(x_t)\right\|^2-\frac{\tilde{\lambda}\eta_l}{4}\left\|g_t\right\|^{2}-\frac{(1-\tilde{\lambda})\eta_l}{4}\left\|m_{t}\right\|^2.
		\end{split}
	\end{equation}
	Upon rearranging the terms in \eqref{eq:f}, summing over $t\in[T]$, and taking the total expectation, we obtain
	\begin{equation}\label{eq:sumnf}
		\begin{split}
			\frac{\eta_l}{2}\E\left[\sum_{t=1}^T\left\|\nabla f(x_t)\right\|^2\right] &\leq f(x_1)-f_\ast+\frac{\tilde{\lambda}\eta_u}{2}\E\left[\sum_{t=1}^T\left\|g_{t}-\nabla f(x_t)\right\|^2\right]\\&+\frac{(1-\tilde{\lambda})\eta_u}{2}\E\left[\sum_{t=1}^T\left\|m_{t}-\nabla f(x_t)\right\|^2\right]\\&-\frac{\tilde{\lambda}\eta_l}{4}\E\left[\sum_{t=1}^T\left\|g_t\right\|^{2}\right]-\frac{(1-\tilde{\lambda})\eta_l}{4}\E\left[\sum_{t=1}^T\left\|m_{t}\right\|^2\right] , 
		\end{split}
	\end{equation}
	where $f_\ast$ is the lower bound of $f$ by Assumption \ref{ass:smooth}. Let $\Delta_t=\left\|m_t-\nabla f(x_t)\right\|^2$. Then, according to Lemma \ref{le:delta}, we have
	\begin{equation}\label{eq:delta}
		\begin{split}
			\E\left[\sum_{t=1}^T\Delta_{t}\right]\leq &\E\left[\frac{\Delta_{1}}{1-\beta}\right]+\left(1-\beta\right)D_0T+2\left(1-\beta\right)D_1\E\left[\sum_{t=1}^T\left\|\nabla f(x_t)\right\|^2\right]\\ &+\tilde{\lambda}\left(2\left(1-\beta\right)D_1L^2\eta_u^2+\frac{\beta^2L^2\eta_u^2}{(1-\beta)^2}\right)\E\left[\sum_{t=1}^T\left\|g_t\right\|^2\right]\\ &+(1-\tilde{\lambda})\left(2\left(1-\beta\right)D_1L^2\eta_u^2+\frac{\beta^2L^2\eta_u^2}{(1-\beta)^2}\right)\E\left[\sum_{t=1}^T\left\|m_{t}\right\|^2\right].
		\end{split}
	\end{equation}
	A substitution of \eqref{eq:delta} into \eqref{eq:sumnf} yields
	\begin{equation}\label{eq:sumnf1}
		\begin{split}
			&\frac{\eta_l}{2}\E\left[\sum_{t=1}^T\left\|\nabla f(x_t)\right\|^2\right]\leq f(x_1)-f_\ast+\frac{(1-\tilde{\lambda})\eta_u\Delta_{1}}{2(1-\beta)}+\frac{(1-\tilde{\lambda})\left(1-\beta\right)D_0\eta_u}{2}T\\
			&\;+(1-\tilde{\lambda})\left(1-\beta\right)D_1\eta_u\E\left[\sum_{t=1}^T\left\|\nabla f(x_t)\right\|^2\right]+\frac{\tilde{\lambda}\eta_u}{2}\E\left[\sum_{t=1}^T\left\|g_{t}-\nabla f(x_t)\right\|^2\right]\\
			&\;+\tilde{\lambda}\left((1-\tilde{\lambda})\left(1-\beta\right)D_1L^2\eta_u^3+\frac{(1-\tilde{\lambda})\beta^2 L^2\eta_u^3}{2\left(1-\beta\right)^2}-\frac{\eta_l}{4}\right)\E\left[\sum_{t=1}^T\left\|g_t\right\|^2\right]\\
			&\;+(1-\tilde{\lambda})\left((1-\tilde{\lambda})\left(1-\beta\right)D_1L^2\eta_u^3+\frac{(1-\tilde{\lambda})\beta^2L^2\eta_u^3}{2\left(1-\beta\right)^2}-\frac{\eta_l}{4}\right)\E\left[\sum_{t=1}^T\left\|m_{t}\right\|^2\right].
		\end{split}
	\end{equation}
	Denote $\Psi=(1-\tilde{\lambda})\left(1-\beta\right)D_1L^2\eta_u^3+\dfrac{(1-\tilde{\lambda})\beta^2L^2\eta_u^3}{2\left(1-\beta\right)^2}-\dfrac{\eta_l}{4}$. Then, the inequality \eqref{eq:sumnf1} can be rearranged as
	\begin{equation}
		\begin{split}
			&\frac{\eta_l}{2}\E\left[\sum_{t=1}^T\left\|\nabla f(x_t)\right\|^2\right]\leq f(x_1)-f_\ast+\frac{(1-\tilde{\lambda})\eta_u\Delta_{1}}{2\left(1-\beta\right)}+\frac{(1-\tilde{\lambda})\left(1-\beta\right)D_0\eta_u}{2}T\\
			&\;\quad+(1-\tilde{\lambda})\left(1-\beta\right)D_1\eta_u\E\left[\sum_{t=1}^T\left\|\nabla f(x_t)\right\|^2\right]+\frac{\tilde{\lambda}\eta_u}{2}\E\left[\sum_{t=1}^T\left\|g_{t}-\nabla f(x_t)\right\|^2\right]\\
			&\;\quad+\tilde{\lambda}\Psi\E\left[\sum_{t=1}^T\left\|g_t\right\|^2\right]+(1-\tilde{\lambda})\Psi\E\left[\sum_{t=1}^T\left\|m_{t}\right\|^2\right].
		\end{split}
	\end{equation}
	By Assumption \ref{ass:weak}, $\E_t\left[\left\|g_t-\nabla f(x_t)\right\|^2\right]\leq D_0+D_1\left\|\nabla f(x_t)\right\|^2$, we have
	\begin{equation}\label{eq:sumnf2}
		\begin{split}
			&\frac{\eta_l}{2}\E\left[\sum_{t=1}^T\left\|\nabla f(x_t)\right\|^2\right]\leq f(x_1)-f_\ast+\left(\frac{(1-\tilde{\lambda})\left(1-\beta\right)D_0\eta_u}{2}+\frac{\tilde{\lambda} D_0\eta_u}{2}\right)T\\&\quad~+\frac{(1-\tilde{\lambda})\eta_u \Delta_{1}}{2\left(1-\beta\right)}+\left((1-\tilde{\lambda})\left(1-\beta\right)D_1\eta_u+\frac{\tilde{\lambda} D_1\eta_u}{2}\right)\E\left[\sum_{t=1}^T\left\|\nabla f(x_t)\right\|^2\right]\\&\quad~+\tilde{\lambda}\Psi\E\left[\sum_{t=1}^T\left\|g_t\right\|^2\right]+(1-\tilde{\lambda})\Psi\E\left[\sum_{t=1}^T\left\|m_{t}\right\|^2\right].
		\end{split}
	\end{equation}
	Since
	\begin{equation}
		\tilde{\lambda}\in[0,1],\quad\beta\in[0,1),\quad\eta_u^3\leq\min\left\{\frac{\eta_l}{8D_1L^2}, \frac{\left(1-\beta\right)^2\eta_l}{4L^2}\right\},
	\end{equation}
	we obtain
	\begin{equation}\label{eq:01}
		(1-\tilde{\lambda})\left(1-\beta\right)D_1L^2\eta_u^3 \leq D_1L^2\eta_u^3 \leq D_1L^2  \frac{\eta_l}{8D_1L^2} = \frac{\eta_l}{8},
	\end{equation}
	\begin{equation}\label{eq:02}
		\frac{(1-\tilde{\lambda})\beta^2L^2\eta_u^3}{2\left(1-\beta\right)^2} \leq \frac{L^2\eta_u^3}{2\left(1-\beta\right)^2} \leq \frac{L^2}{2\left(1-\beta\right)^2} \frac{\left(1-\beta\right)^2\eta_l}{4L^2} = \frac{\eta_l}{8}.
	\end{equation}
	Upon combining \eqref{eq:01} and \eqref{eq:02}, it is straightforward to see that
	\begin{equation}
		\begin{split}
			\Psi&=(1-\tilde{\lambda})\left(1-\beta\right)D_1L^2\eta_u^3+\frac{(1-\tilde{\lambda})\beta^2L^2\eta_u^3}{2\left(1-\beta\right)^2}-\frac{\eta_l}{4} \\& \leq\frac{\eta_l}{8}+\frac{\eta_l}{8}-\frac{\eta_l}{4}=0.
		\end{split}
	\end{equation}
	In this way, \eqref{eq:sumnf2} reduces to
	\begin{equation}\label{eq:sumnf3}
		\begin{split}
			&\frac{\eta_l}{2}\E\left[\sum_{t=1}^T\left\|\nabla f(x_t)\right\|^2\right]\leq f(x_1)-f_\ast\\&\qquad\qquad\quad+\left(\frac{(1-\tilde{\lambda})\left(1-\beta\right)D_0\eta_u}{2}+\frac{\tilde{\lambda} D_0\eta_u}{2}\right)T  +\frac{(1-\tilde{\lambda})\eta_u \Delta_{1}}{2\left(1-\beta\right)}\\&\qquad\qquad\quad+\left((1-\tilde{\lambda})\left(1-\beta\right)D_1\eta_u+\frac{\tilde{\lambda} D_1\eta_u}{2}\right)\E\left[\sum_{t=1}^T\left\|\nabla f(x_t)\right\|^2\right].
		\end{split}
	\end{equation}
	Since 
	\begin{equation}
		\tilde{\lambda}=(1-\beta)\lambda\in[0,1],\quad 1-\beta\leq\min\left\{\frac{\eta_l}{2(2+\lambda) D_1\eta_u},1\right\},
	\end{equation}
	we have
	\begin{equation}\label{eq:11}
		\begin{split}
			(1-\tilde{\lambda})\left(1-\beta\right)D_1\eta_u+\frac{\tilde{\lambda} D_1\eta_u}{2}  &\leq \frac{(1-\beta)(2+\lambda) D_1\eta_u}{2} \\& \leq \frac{(2+\lambda) D_1\eta_u}{2}  \frac{\eta_l}{2(2+\lambda)D_1\eta_u} =\frac{\eta_l}{4},
		\end{split}
	\end{equation}
	and 
	\begin{equation}\label{eq:12}
		\begin{split}
			\frac{(1-\tilde{\lambda})\left(1-\beta\right)D_0\eta_u}{2}+\frac{\tilde{\lambda} D_0\eta_u}{2} &\leq \frac{D_0\eta_u}{2}\left(\left(1-\beta\right)+\tilde{\lambda}\right) \\& = \frac{\left(1-\beta\right)\left(1+\lambda\right)D_0\eta_u}{2}.
		\end{split}
	\end{equation}
	Finally, after plugging \eqref{eq:11} and \eqref{eq:12} back into \eqref{eq:sumnf3}, and rearranging the terms, we obtain
	\begin{equation}\label{eq:sumnf4}
		\begin{split}
			\frac{\eta_l}{4}\E\left[\sum_{t=1}^T\left\|\nabla f(x_t)\right\|^2\right]\leq& f(x_1)-f_\ast+\frac{(1-\tilde{\lambda})\eta_u \Delta_{1}}{2\left(1-\beta\right)}\\& +\frac{\left(1-\beta\right)\left(1+\lambda\right)D_0\eta_u}{2} T.
		\end{split}
	\end{equation}
	A multiplication of both sides of \eqref{eq:sumnf4} by $\dfrac{4}{\eta_l T}$ gives
	\begin{equation}
		\begin{split}
			\E\left[\frac{1}{T}\sum_{t=1}^T\left\|\nabla f(x_t)\right\|^2\right]\leq& \frac{4\left(f(x_1)-f_\ast\right)}{\eta_l T}+\frac{2(1-\tilde{\lambda})\eta_u \Delta_{1}}{\left(1-\beta\right)\eta_l T}\\&+\frac{2\left(1-\beta\right)\left(1+\lambda\right)D_0\eta_u}{\eta_l}.
		\end{split}
	\end{equation}
	This completes the proof.
\end{proof}

From Theorem \ref{th:UAdam}, observe that UAdam converges to the neighborhood of stationary point and the size of neighborhood decreases as $\beta$ increases. In particular, when the strong growth condition ($D_0=0$) and $\eta_u=K\eta_l$ $(K>1)$ are satisfied, we obtain the following corollary.

\begin{corollary}\label{co:k}
	Suppose that the conditions in Theorem \ref{th:UAdam} hold for UAdam (see Algorithm \ref{UAdam}). With $\eta_u\leq\min\left\{\frac{1}{2L\sqrt{2KD_1}}, \frac{1-\beta}{2L\sqrt{K}},\frac{1}{2KL}\right\}$ and $0<1-\beta\leq\min\left\{\frac{1}{2(2+\lambda) D_1K},1\right\}$, we have
	\begin{equation}
		\begin{split}
			\frac{1}{T}\E\left[\sum_{t=1}^T\left\|\nabla f(x_t)\right\|^2\right]\leq\mathcal{O}\left(\frac{1}{T}\right).
		\end{split}
	\end{equation}
\end{corollary}

\begin{proof}
	From Theorem \ref{th:UAdam}, since $\eta_u=K\eta_l,K>1$, then for $\eta_u$, we have
	\begin{equation}
		\begin{split}
			&\eta_u^3\leq\frac{\eta_l}{8D_1L^2}=\frac{\eta_u}{8KD_1L^2} \quad\Leftrightarrow\quad \eta_u\leq\frac{1}{2L\sqrt{2KD_1}}, \\
			&\eta_u^3\leq\frac{\left(1-\beta\right)^2\eta_l}{4L^2}=\frac{\left(1-\beta\right)^2\eta_u}{4KL^2} \quad\Leftrightarrow\quad \eta_u\leq\frac{1-\beta}{2L\sqrt{K}}, \\
			&\eta_u^2\leq\frac{\eta_l}{2L}=\frac{\eta_u}{2KL} \quad\Leftrightarrow\quad \eta_u\leq\frac{1}{2KL}.
		\end{split}
	\end{equation}
	Furthermore, since $\frac{\eta_l}{\eta_u}=\frac{1}{K}$, the range of $\beta$ becomes
	\begin{equation}
		0<1-\beta\leq\min\left\{\frac{1}{2(2+\lambda) D_1K},1\right\}.
	\end{equation}
	This completes the proof.
\end{proof}

\begin{remark}
	From Algorithm \ref{UAdam}, we can see that when $\lambda=0$, UAdam degenerates into an Adam-type algorithm, while when $\lambda=1$, UAdam reduces to an NAdam-type algorithm. Then, from Corollary \ref{co:k}, we can directly obtain the convergence for the Adam-type and NAdam-type algorithms. This demonstrates the power and generality of Corollary \ref{co:k}, which allows us to immediately obtain the convergence results of many popular deep learning algorithms, such as AMSGrad, AdaBound, AdaFom, and Adan, to mention but a few. This is consistent with the convergence results in \cite{guo,Adan}. Last but not least, we can obtain a faster convergence rate than existing convergence results in \cite{Chen,1,Adam-without-modification,zou}, which is attributed the setting of Assumption \ref{ass:bound}.
\end{remark}

\section{Conclusion}\label{sec:conclusion}

We have proposed a novel unified framework for the design and analysis of adaptive momentum optimizers in deep learning. This unifying platform, referred to as the unified Adam (UAdam), combines unified momentum methods, including SHB and SNAG, with a class of adaptive learning rate algorithms satisfying a boundedness condition. By using the variance recursion of the stochastic gradient moving average estimator, we have established that UAdam can converge to the neighborhood of stationary points with the rate of $\mathcal{O}(1/T)$ in smooth non-convex settings and that the size of neighborhood decreases as $\beta$ increases. Under an extra condition (strong growth condition), we have further obtained that Adam converges to stationary points. These results have implied that, for a given problem in hand, with appropriate hyperparameter selection Adam can converge without any modification on its update rules. In addition, our analysis of UAdam does not impose any restrictions on the second-order moment parameter, $\beta_2$, and only requires a sufficiently large first-order momentum parameter (close to 1), which is in line with the hyperparameter settings in practice. The analysis has provided new insights into the convergence of Adam and NAdam, and a unifying platform for the development of new algorithms in this setting. Future work will investigate the convergence of Adam under biased gradient conditions.

\begin{appendices}

	\section{Equivalence form of SNAG}\label{app:NAG}
	
	\begin{proposition}\label{pro:NAG}
		Let $\bar{x}_t$ and $m_t$ denote respectively the iteration and momentum of the original SNAG, the update of which is given by
		\begin{equation}\label{NAG1}
			\text{SNAG}_1\text{: }
			\left\{
			\begin{split}
				&\bar{m}_{t}=\beta \bar{m}_{t-1}-\alpha\nabla f\left(\bar{x}_t+\beta \bar{m}_{t-1},\xi_t\right)\\
				&\bar{x}_{t+1}=\bar{x}_t+\bar{m}_t
			\end{split}
			\right..
		\end{equation}
		Then, when $\alpha=\eta(1-\beta)$ and $\bar{m}_t=-\eta m_t$, SNAG$_1$ is equivalent to
		\begin{equation}\label{NAG2}
			\text{SNAG}_2\text{: }
			\left\{
			\begin{split}
				&m_t=\beta m_{t-1}+\left(1-\beta\right) g_t \\
				&x_{t+1} = x_{t}-\eta\beta m_t-\eta\left(1-\beta\right) g_t
			\end{split}
			\right..
		\end{equation}
		
	\end{proposition}
	
	\begin{proof}
		Define $x_t=\bar{x}_t+\beta \bar{m}_{t-1}$ and $g_t=\nabla f\left(x_t,\xi_t\right)$. Then, the first identity in \eqref{NAG1} becomes
		\begin{equation}\label{eq:a1}
			\bar{m}_{t}=\beta \bar{m}_{t-1}-\alpha\nabla f\left(x_t,\xi_t\right)=\beta \bar{m}_{t-1}-\alpha g_t.
		\end{equation}
		Since $\alpha=\eta(1-\beta)$ and $\bar{m}_t=-\eta m_t$, then \eqref{eq:a1} becomes
		\begin{equation}\label{eq:a2}
			\begin{split}
				m_t = \frac{-\beta \bar{m}_{t-1}+\alpha g_t}{\eta}=\frac{\beta \eta m_{t-1}+\eta(1-\beta) g_t}{\eta}=\beta m_{t-1}+(1-\beta)g_t.
			\end{split}
		\end{equation}
		Recalling that $x_t=\bar{x}_t+\beta \bar{m}_{t-1}$, we obtain
		\begin{equation}
			\begin{split}
				x_{t+1}-x_{t}& \;=\,\,\bar{x}_{t+1}+\beta \bar{m}_{t}-\bar{x}_t-\beta \bar{m}_{t-1}\\
				&\overset{\eqref{NAG1}}{=}\bar{m}_t+\beta \bar{m}_{t}-\beta \bar{m}_{t-1}\\
				& \;=\,\, -\eta m_t-\eta\beta m_{t}+\eta\beta m_{t-1}\\
				& \overset{\eqref{eq:a2}}{=}\, -\eta m_t-\eta\beta m_{t}+\eta\left(m_t-\left(1-\beta\right) g_t\right)\\
				& \;=\,\, -\eta\beta m_{t}-\eta\left(1-\beta\right) g_t,
			\end{split}
		\end{equation}
		where the third equality follows from $\bar{m}_t=-\eta m_t$. 
		Therefore, the final equivalence form of the original SNAG$_1$ becomes
		\begin{equation}
			\left\{
			\begin{split}
				&m_t=\beta m_{t-1}+\left(1-\beta\right) g_t\\
				&x_{t+1} = x_{t}-\eta\beta m_{t}-\eta\left(1-\beta\right) g_t
			\end{split}
			\right..
		\end{equation}
		This completes the proof.
	\end{proof}

	\begin{proposition}\label{pro:NME}
		Xie $et$ $al$. \cite{Adan} proposed a Nesterov momentum estimation (NME) method as follows
		\begin{equation}\label{NME}
			\text{NME: }
			\left\{
			\begin{split}
				&\bar{m}_t=\beta \bar{m}_{t-1}+(1-\beta)\left(g_t+\beta\left(g_t-g_{t-1}\right)\right) \\
				&x_{t+1}=x_t-\eta \bar{m}_t
			\end{split}
			\right..
		\end{equation}
		Then, NME is equivalent to 
		\begin{equation}\label{SNAG}
			\text{SNAG: }
			\left\{
			\begin{split}
				&m_t=\beta m_{t-1}+\left(1-\beta\right) g_t \\
				&x_{t+1} = x_{t}-\eta\beta m_t-\eta\left(1-\beta\right) g_t
			\end{split}
			\right..
		\end{equation}

	\end{proposition}
	
	\begin{proof}
		According to SNAG, let $\bar{m}_t=\beta m_t+\left(1-\beta\right) g_t$, the second equality of \eqref{SNAG} becomes
		\begin{equation}
			x_{t+1}=x_t-\eta \bar{m}_t.
		\end{equation}
		According to the definition of $\bar{m}_t$, we have
		\begin{equation}
			\begin{split}
				\bar{m}_t-\beta \bar{m}_{t-1} &\;\,=\; \beta m_t+\left(1-\beta\right) g_t-\beta \left(\beta m_{t-1}+\left(1-\beta\right) g_{t-1}\right) \\
				&\;\,=\; \beta \left(m_t-\beta m_{t-1}\right)+\left(1-\beta\right) g_t-\beta \left(1-\beta\right) g_{t-1} \\
				&\overset{\eqref{SNAG}}{=} \beta \left(1-\beta\right) g_t+\left(1-\beta\right) g_t-\beta \left(1-\beta\right) g_{t-1} \\
				&\;\,=\; (1-\beta)\left(g_t+\beta\left(g_t-g_{t-1}\right)\right).
			\end{split}
		\end{equation}
		Consequently, NME is equivalent to SNAG. 
	\end{proof}
	
	\section{Equivalence relationship of SUM}\label{app:SUM}
	
	\begin{proposition}\label{pro:SUM}
		Liu $et$ $al$. \cite{SUM} unified SHB and SNAG as follows
		\begin{equation}\label{SUM1}
			\text{SUM}_1\text{: }
			\left\{
			\begin{split}
				&m_t=\mu m_{t-1}-\eta_t g_t\\
				&x_{t+1} = x_{t}-\lambda\eta_t g_t+(1-\tilde{\lambda})m_t
			\end{split}
			\right.,
		\end{equation}
		where $\tilde{\lambda}:=(1-\mu)\lambda\in[0,1]$. When $\eta_t=\eta\left(1-\beta\right)$ and $\mu=\beta$, SUM$_{1}$ is equivalent to the following unified momentum method 
		\begin{equation}\label{SUM2}
			\text{SUM}_2\text{: }
			\left\{
			\begin{split}
				&m_t=\beta m_{t-1}+(1-\beta) g_t \\
				&\bar{m}_t=m_t-\tilde{\lambda} (m_t-g_t) \\
				&x_{t+1} = x_{t}-\eta \bar{m}_t
			\end{split}
			\right.,
		\end{equation}
		where $\tilde{\lambda}=\left(1-\beta\right)\lambda\in[0,1]$ and $\beta\in[0,1)$.			
	\end{proposition}
	
	\begin{proof}
		First, SUM$_{1}$ can be written as
		\begin{equation}\label{eq:SUM11}
			\begin{split}
				x_{t+1} &\overset{\eqref{SUM1}}{=} x_{t}-\lambda\eta_t g_t+(1-\tilde{\lambda})m_t \\
				&\overset{\eqref{SUM1}}{=}
				x_{t}-\lambda\eta_t g_t+(1-\tilde{\lambda})\left(\mu m_{t-1}-\eta_t g_t\right) \\
				& \,\,\,=\,\,
				x_{t}-\lambda\eta_t g_t-(1-\tilde{\lambda})\eta_t g_t+\mu(1-\tilde{\lambda})m_{t-1}\\
				&\overset{\eqref{SUM1}}{=} x_{t}-\lambda\eta_t g_t-(1-\tilde{\lambda})\eta_t g_t+\mu\left(x_{t}-x_{t-1}+\lambda\eta_{t-1} g_{t-1}\right) \\
				& \,\,\,=\,\, x_{t}-\lambda\eta_t g_t-(1-(1-\mu)\lambda)\eta_t g_t+\mu\left(x_{t}-x_{t-1}+\lambda\eta_{t-1} g_{t-1}\right).
			\end{split}
		\end{equation}
		In a similar manner, SUM$_{2}$ can be written as
		\begin{equation}\label{eq:SUM21}
			\begin{split}
				x_{t+1} &\overset{\eqref{SUM2}}{=} x_{t}-\eta \bar{m}_t \\
				&\overset{\eqref{SUM2}}{=}x_{t}-\eta(m_t-\tilde{\lambda} (m_t-g_t)) \\
				& \,\,\,=\,\,\,x_{t}-\eta\tilde{\lambda}g_t-\eta(1-\tilde{\lambda})m_t .
			\end{split}
		\end{equation}	
		Upon substituting the first equality of \eqref{SUM2} into the last term of \eqref{eq:SUM21}, we obtain
		\begin{equation}\label{eq:SUM23}
			\begin{split}
				x_{t+1} 
				&\; =\,\,
				x_{t}-\eta\tilde{\lambda} g_t+\eta(1-\tilde{\lambda}) (\beta m_{t-1}+\left(1-\beta\right) g_t)\\
				&\; =\,\,
				x_{t}-\eta\tilde{\lambda} g_t-\eta(1-\tilde{\lambda})(1-\beta) g_t-\eta(1-\tilde{\lambda})\beta m_{t-1}\\
				&\overset{\eqref{eq:SUM21}}{=} x_{t}-\eta\tilde{\lambda}g_t-\eta(1-\beta)(1-\tilde{\lambda})g_t+\beta\left(x_{t}-x_{t-1}+\eta\tilde{\lambda} g_{t-1}\right)\\
				&\; =\,\, x_{t}-\eta\left(1-\beta\right)\lambda g_t-\eta(1-\beta)(1-\left(1-\beta\right)\lambda)g_t\\&\qquad+\beta\left(x_{t}-x_{t-1}+\eta\left(1-\beta\right)\lambda g_{t-1}\right).
			\end{split}
		\end{equation}			
		By comparing the coefficients in \eqref{eq:SUM11} and \eqref{eq:SUM23}, it is straightforward to observe that SUM$_{2}$ and SUM$_{1}$ are equivalent, when $\eta_t=\eta\left(1-\beta\right)$ and $\mu=\beta$.
	\end{proof}

\end{appendices}

\bibliographystyle{apalike}
\bibliography{References}

\begin{thebibliography}{}

\bibitem[Bottou et~al., 2018]{siam}
Bottou, L., Curtis, F.~E., and Nocedal, J. (2018).
\newblock Optimization methods for large-scale machine learning.
\newblock {\em SIAM Review}, 60(2):223--311.

\bibitem[Chen et~al., 2022]{Towards}
Chen, C., Shen, L., Zou, F., and Liu, W. (2022).
\newblock Towards practical {A}dam: Non-convexity, convergence theory, and
  mini-batch acceleration.
\newblock {\em Journal of Machine Learning Research}, 23(229):1--47.

\bibitem[Chen et~al., 2019]{Chen}
Chen, X., Liu, S., Sun, R., and Hong, M. (2019).
\newblock On the convergence of a class of {A}dam-type algorithms for
  non-convex optimization.
\newblock In {\em Proceedings of the International Conference on Learning
  Representations}, pages 1--30.

\bibitem[D{\'e}fossez et~al., 2022]{1}
D{\'e}fossez, A., Bottou, L., Bach, F., and Usunier, N. (2022).
\newblock A simple convergence proof of {A}dam and {A}dagrad.
\newblock {\em Transactions on Machine Learning Research}, pages 1--30.

\bibitem[Dosovitskiy et~al., 2021]{computer-version}
Dosovitskiy, A., Beyer, L., Kolesnikov, A., Weissenborn, D., Zhai, X.,
  Unterthiner, T., Dehghani, M., Minderer, M., Heigold, G., Gelly, S.,
  Uszkoreit, J., and Houlsby, N. (2021).
\newblock An image is worth 16x16 words: Transformers for image recognition at
  scale.
\newblock In {\em Proccedings of the International Conference on Learning
  Representations}, pages 1--21.

\bibitem[Dozat, 2016]{NAdam}
Dozat, T. (2016).
\newblock Incorporating {N}esterov momentum into {A}dam.
\newblock In {\em Proceedings of the International Conference on Learning
  Representations}, pages 1--4.

\bibitem[Duchi et~al., 2011]{Adagrad}
Duchi, J., Hazan, E., and Singer, Y. (2011).
\newblock Adaptive subgradient methods for online learning and stochastic
  optimization.
\newblock {\em Journal of Machine Learning Research}, 12:2121--2159.

\bibitem[Ghadimi and Lan, 2013]{SGD2}
Ghadimi, S. and Lan, G. (2013).
\newblock Stochastic first- and zeroth-order methods for nonconvex stochastic
  programming.
\newblock {\em SIAM Journal on Optimization}, 23(4):2341--2368.

\bibitem[Guo et~al., 2021]{guo}
Guo, Z., Xu, Y., Yin, W., Jin, R., and Yang, T. (2021).
\newblock A novel convergence analysis for algorithms of the {A}dam family.
\newblock In {\em Proceedings of the 13th Annual Workshop on Optimization for
  Machine Learning}, pages 1--13.

\bibitem[He et~al., 2016]{imagerecognition}
He, K., Zhang, X., Ren, S., and Sun, J. (2016).
\newblock Deep residual learning for image recognition.
\newblock In {\em Proceedings of the IEEE Conference on Computer Vision and
  Pattern Recognition (CVPR)}, pages 770--778.

\bibitem[Howard and Ruder, 2018]{text-classification}
Howard, J. and Ruder, S. (2018).
\newblock Universal language model fine-tuning for text classification.
\newblock In {\em Proceedings of the 56th Annual Meeting of the Association for
  Computational Linguistics}, pages 328--339.

\bibitem[Khaled and Richt{\'a}rik, 2023]{better-sgd}
Khaled, A. and Richt{\'a}rik, P. (2023).
\newblock Better theory for {SGD} in the nonconvex world.
\newblock {\em Transactions on Machine Learning Research}, pages 1--32.

\bibitem[Kingma and Ba, 2015]{Adam}
Kingma, D. and Ba, J. (2015).
\newblock Adam: {A} method for stochastic optimization.
\newblock In {\em Proceedings of the 3rd International Conference on Learning
  Representations}, pages 1--15.

\bibitem[Li and Orabona, 2019]{Li}
Li, X. and Orabona, F. (2019).
\newblock On the convergence of stochastic gradient descent with adaptive
  stepsizes.
\newblock In {\em Proceedings of the Twenty-Second International Conference on
  Artificial Intelligence and Statistics}, volume~89 of {\em Proceedings of
  Machine Learning Research}, pages 983--992.

\bibitem[Lin et~al., 2020]{Linzhouchen}
Lin, Z., Li, H., and Fang, C. (2020).
\newblock {\em Accelerated Optimization for Machine Learning}.
\newblock Springer.

\bibitem[Liu et~al., 2022]{xuadabound}
Liu, J., Kong, J., Xu, D., Qi, M., and Lu, Y. (2022).
\newblock Convergence analysis of {A}da{B}ound with relaxed bound functions for
  non-convex optimization.
\newblock {\em Neural Networks}, 145:300--307.

\bibitem[Liu et~al., 2023]{SUM}
Liu, J., Xu, D., Lu, Y., Kong, J., and Mandic, D.~P. (2023).
\newblock Last-iterate convergence analysis of stochastic momentum methods for
  neural networks.
\newblock {\em Neurocomputing}, 527:27--35.

\bibitem[Luo et~al., 2019]{AdaBound}
Luo, L., Xiong, Y., Liu, Y., and Sun, X. (2019).
\newblock Adaptive gradient methods with dynamic bound of learning rate.
\newblock In {\em Proceedings of the 7th International Conference on Learning
  Representations}, pages 1--21.

\bibitem[Nesterov, 1983]{NAG}
Nesterov, Y. (1983).
\newblock A method of solving a convex programming problem with convergence
  rate $\mathcal{O}(1/k^2)$.
\newblock {\em Soviet Mathematics Doklady}, 27:372--376.

\bibitem[Nesterov, 2013]{Nesterov}
Nesterov, Y. (2013).
\newblock {\em Introductory Lectures on Convex Optimization: A Basic Course}.
\newblock Springer Science and Business Media.

\bibitem[Polyak, 1964]{HB}
Polyak, B.~T. (1964).
\newblock Some methods of speeding up the convergence of iteration methods.
\newblock {\em USSR Computational Mathematics and Mathematical Physics},
  4(5):1--17.

\bibitem[Reddi et~al., 2018]{Reddi}
Reddi, S.~J., Kale, S., and Kumar, S. (2018).
\newblock On the convergence of {A}dam and beyond.
\newblock In {\em Proceedings of the International Conference on Learning
  Representations (ICLR)}, pages 1--23.

\bibitem[Robbins, 1951]{SGD1}
Robbins, H.~E. (1951).
\newblock A stochastic approximation method.
\newblock {\em Annals of Mathematical Statistics}, 22:400--407.

\bibitem[Sun, 2020]{sun}
Sun, R.~Y. (2020).
\newblock Optimization for deep learning: An overview.
\newblock {\em Journal of the Operations Research Society of China},
  8(2):249--294.

\bibitem[Sutskever et~al., 2013]{NAG2}
Sutskever, I., Martens, J., Dahl, G., and Hinton, G. (2013).
\newblock On the importance of initialization and momentum in deep learning.
\newblock In {\em Proceedings of the International Conference on Machine
  Learning}, pages 1139--1147.

\bibitem[Tieleman and Hinton, 2012]{RMSprop}
Tieleman, T. and Hinton, G. (2012).
\newblock Lecture 6.5-{RMSP}rop: Divide the gradient by a running average of
  its recent magnitude.
\newblock {\em COURSERA: Neural Networks for Machine Learning}.

\bibitem[Tong et~al., 2022]{sadam}
Tong, Q., Liang, G., and Bi, J. (2022).
\newblock Calibrating the adaptive learning rate to improve convergence of
  {A}dam.
\newblock {\em Neurocomputing}, 481:333--356.

\bibitem[Wang et~al., 2022]{L0L1}
Wang, B., Zhang, Y., Zhang, H., Meng, Q., Ma, Z., Liu, T., and Chen, W. (2022).
\newblock Provable adaptivity in {A}dam.
\newblock {\em arXiv preprint arXiv:2208.09900}.

\bibitem[Wang et~al., 2017]{Wangmengdi}
Wang, M., Fang, E., and Liu, H. (2017).
\newblock Stochastic compositional gradient descent: {A}lgorithms for
  minimizing compositions of expected-value functions.
\newblock {\em Mathematical Programming}, 161:419--449.

\bibitem[Wani et~al., 2020]{Speech-Emotion-Recognition}
Wani, T.~M., Gunawan, T.~S., Qadri, S. A.~A., Mansor, H., Kartiwi, M., and
  Ismail, N. (2020).
\newblock Speech emotion recognition using convolution neural networks and deep
  stride convolutional neural networks.
\newblock In {\em Proceedings of the 6th International Conference on Wireless
  and Telematics (ICWT)}, pages 1--6.

\bibitem[Ward et~al., 2020]{Ward}
Ward, R., Wu, X., and Bottou, L. (2020).
\newblock Adagrad stepsizes: Sharp convergence over nonconvex landscapes.
\newblock {\em Journal of Machine Learning Research}, 21(219):1--30.

\bibitem[Xie et~al., 2022]{Adan}
Xie, X., Zhou, P., Li, H., Lin, Z., and Yan, S. (2022).
\newblock Adan: Adaptive {N}esterov momentum algorithm for faster optimizing
  deep models.
\newblock {\em arXiv preprint arXiv:2208.06677}.

\bibitem[Xu et~al., 2021]{rmspropw}
Xu, D., Zhang, S., Zhang, H., and Mandic, D.~P. (2021).
\newblock Convergence of the {RMSP}rop deep learning method with penalty for
  nonconvex optimization.
\newblock {\em Neural Networks}, 139:17--23.

\bibitem[Zaheer et~al., 2018]{Yogi}
Zaheer, M., Reddi, S.~J., Sachan, D.~S., Kale, S., and Kumar, S. (2018).
\newblock Adaptive methods for nonconvex optimization.
\newblock In {\em Proceedings of the Advances in Neural Information Processing
  Systems}, volume~31, pages 9815--9825.

\bibitem[Zhang et~al., 2022]{Adam-without-modification}
Zhang, Y., Chen, C., Shi, N., Sun, R., and Luo, Z. (2022).
\newblock Adam can converge without any modification on update rules.
\newblock In {\em Proceedings of the Advances in Neural Information Processing
  Systems}, pages 1--14.

\bibitem[Zhou et~al., 2020]{Zhou}
Zhou, D., Tang, Y., Yang, Z., Cao, Y., and Gu, Q. (2020).
\newblock On the convergence of adaptive gradient methods for nonconvex
  optimization.
\newblock In {\em Proceedings of the 12th Annual Workshop on Optimization for
  Machine Learning}, pages 1--25.

\bibitem[Zou et~al., 2018]{AdaUSM}
Zou, F., Shen, L., Jie, Z., Sun, J., and Liu, W. (2018).
\newblock Weighted {A}dagrad with unified momentum.
\newblock {\em arXiv preprint arXiv:1808.03408}.

\bibitem[Zou et~al., 2019]{zou}
Zou, F., Shen, L., Jie, Z., Zhang, W., and Liu, W. (2019).
\newblock A sufficient condition for convergences of {A}dam and {RMSP}rop.
\newblock In {\em Proceedings of the IEEE/CVF Conference on Computer Vision and
  Pattern Recognition (CVPR)}, pages 11119--11127.

\end{thebibliography}

\end{document}